\documentclass[a4paper]{svproc}
\usepackage{url}

\usepackage{subcaption}
\usepackage{amsmath}
\usepackage{mathabx}
\usepackage[bookmarks=true]{hyperref}
\usepackage{graphicx}
\usepackage{times}
\usepackage{underscore}
\usepackage{moreverb,url}
\usepackage[dvipsnames]{xcolor}
\usepackage{url} 
 
\usepackage{times} 
\usepackage{fixltx2e} 
\usepackage{multicol} 
\usepackage{graphicx} 
\usepackage[utf8]{inputenc} 
\usepackage{graphics}
\usepackage{epsfig} 
\usepackage{bm} 
\usepackage{caption} 
\usepackage{subcaption} 
\usepackage{ifthen} 
\newboolean{include-notes} 
\usepackage[linesnumbered,ruled,noend]{algorithm2e} 
\usepackage[noend]{algpseudocode} 
\usepackage{bm}
\usepackage[sort,numbers, sectionbib]{natbib}

\usepackage[bookmarks=true]{hyperref}
\usepackage{amsmath}
\usepackage{mathabx}
\usepackage{multicol}
\usepackage{lipsum}  
\usepackage{dblfloatfix}
\usepackage[english]{babel}
\usepackage{blindtext}
\usepackage{mathtools}
\usepackage{booktabs}

\renewcommand*{\bibname}{References}
\renewcommand\vec{\mathbf}
\SetKwRepeat{Do}{do}{while}
\begin{document}
\mainmatter             
\title{Multi-Object Grasping in the Plane} 
\titlerunning{Multi-Object Grasping in the Plane} 

\institute{School of Computing, University of Leeds,  UK.\\
\and
The AUTOLab at UC Berkeley (automation.berkeley.edu), USA.}

\author{Wisdom C. Agboh\inst{1,2} \and Jeffrey Ichnowski\inst{2} \and Ken Goldberg\inst{2} \and Mehmet R. Dogar\inst{1}}
\authorrunning{Agboh et al.}
\maketitle  

\begin{abstract}
We consider a novel problem where multiple rigid convex polygonal objects rest in randomly placed positions and orientations on a planar surface visible from an overhead camera. The objective is to efficiently grasp and transport all objects into a bin using multi-object push-grasps, where multiple objects are pushed together to facilitate multi-object grasping. We provide necessary conditions for frictionless multi-object push-grasps and apply these to filter inadmissible grasps in a novel multi-object grasp planner. We find that our planner is 19 times faster than a Mujoco simulator baseline. We also propose a picking algorithm that uses both single- and multi-object grasps to pick objects. In physical grasping experiments comparing performance with a single-object picking baseline, we find that the frictionless multi-object grasping system achieves 13.6\% higher grasp success and is 59.9\% faster, from 212 PPH to 340 PPH. See \url{https://sites.google.com/view/multi-object-grasping} for videos and code. 

\end{abstract}
\begin{figure}[b]
\centering
	   \begin{subfigure}[b]{0.325\textwidth}
	 \centering 
		\includegraphics[scale=0.2]{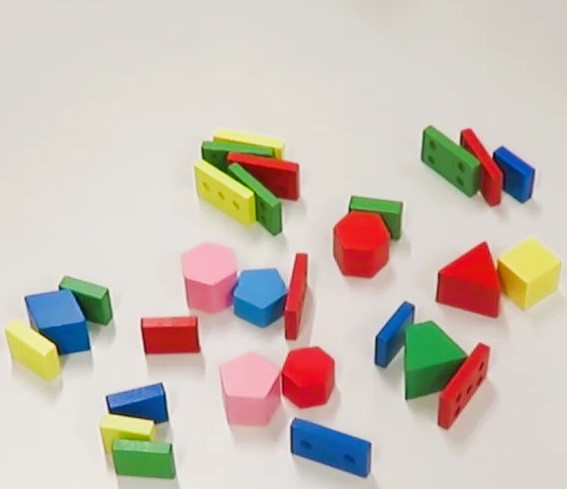}
	\end{subfigure}
		\begin{subfigure}[b]{0.325\textwidth}
	 \centering 
		\includegraphics[scale=0.2]{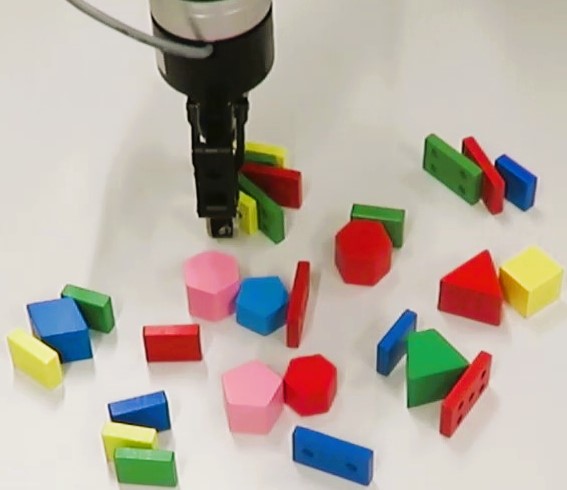}
	\end{subfigure}
	   \begin{subfigure}[b]{0.325\textwidth}
	 \centering 
		\includegraphics[scale=0.2]{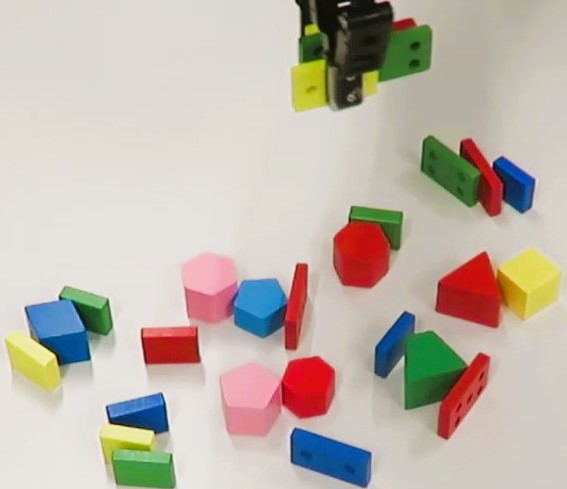}
	\end{subfigure}
	\caption{Robot picks five objects at once with a multi-object push-grasp.}
	\label{fig:mog_intro}
\end{figure}
\section{Introduction}
\label{sec:introduction}
\vspace{-2mm}
There has been a recent increase in demand for fast and efficient robot picking systems, especially for warehouses \cite{Danielczuk-CASE-2018, mahler2017binpicking, Ryo-IROS-2019, Huang-IROS-2021}. State-of-the-art robot grasping systems use only single-object grasps. As robot motion is increasingly becoming the bottleneck of pick-and-place systems~\cite{GOMP,DJ-GOMP}, reducing the number of arm motions with multi-object grasps has the potential to increase speed significantly. Furthermore, in cluttered scenes like Fig.~\ref{fig:mog_intro}, some of the objects may not have the necessary free space around them for the gripper to pick them individually in contrast to multi-object grasps. 

Prior-work \cite{Yamada-ROBIO-2009,  Yamada-ICMA-2012, Harada-IROS-1998} studied multi-object grasps for the case where objects are already in contact and ready to be grasped. However, objects are oftentimes apart and must be pushed/squeezed together before a multi-object grasp is possible. We call these \emph{multi-object push-grasps}. They require physics-based predictions of how objects will move when the grippers close, e.g. to predict whether the objects will remain inside the grasp or whether they will slide out.

One approach to synthesizing such multi-object push-grasps is to divide it into two steps. \citet{Sakamoto-IROS-2021} proposed a system to pick two cuboids at once. Their approach pushes and aligns the objects first before grasping them. Such a two-step approach may not be feasible in cluttered scenes with more than two objects where object access is limited. 

In this work, we propose a single-step multi-object push-grasp where objects are pushed and grasped simultaneously as the robot closes its gripper. One way to achieve that is to create models of objects in a physics simulator, generate single-step candidate grasps, and test them until a successful grasp is found. However, this is computationally expensive given the large number of physics simulations that could be required to find a grasp \cite{Agboh-ISRR-2019, Agboh-CVS-2020, Mohammed-ICRA-2020}. 

Therefore, we propose necessary conditions for multi-object push-grasping. We use these necessary conditions in a novel multi-object grasp planner to rank and filter grasp candidates, before testing them in a physics simulator. The result is a multi-object grasp planner that, on average, tests only two grasp candidates in a physics simulator before finding a successful one. In addition, we propose a picking system that uses the multi-object grasp planner to pick single or multiple objects at each time.

In simulation experiments, we find that our multi-object grasp planner is 19 times faster than a physics simulator baseline. In physical picking experiments compared to a single object picking baseline, we find that the multi-object grasping system achieves 13.6\% higher grasp success, and is 59.9\% faster. In summary, this work contributes:
\begin{itemize}
    \item[1.] Necessary conditions for frictionless multi-object push-grasping. 
    \item[2.] Theorems and proofs on the necessary conditions.
    \item[3.] A multi-object grasp planner. 
    \item[4.] A picking system that uses both single and multi-object grasps.
    \item[5.] Simulation experiments that compare our grasp planner with a physics simulator baseline and physical experiments that compare the picking system with a single object grasping baseline.
\end{itemize}

\section{Related work}
\label{sec:related_work}
\vspace{-2mm}
Picking multiple objects from a table or a bin is common for humans, e.g. waiters. Most prior work \cite{mahler2017binpicking,Morrison-IJRR-2020, Lou-ICRA-2021} in robotic picking focus on single object grasping. Recently, \citet{Sakamoto-IROS-2021} proposed a picking system that uses both one- and two-object grasps to pick cuboids. In this work, we propose a picking system that can grasp an arbitrary number of extruded convex polygonal objects at once. 

A key question here is how to synthesize or generate grasps. 
There is a large body of work on synthesizing single-object grasps. They typically fall into one of two categories --- analytic or data driven. The analytic methods~\cite{Prattichizzo-Handbook-2008, Rodriguez-IJRR-2012, Kehoe-ICRA-2013} assume known object model and contact locations. They find grasps that can resist external wrenches or constrain the object's motion. On the other hand, data-driven approaches~\cite{Bohg-TRO-2014, Goldfeder-AutonRobot-2011, mahler2016dex, Pauly-Frontiers-2021, Bejjani-IROS-2021} generate grasping models that map directly from sensor readings like an RGB-D image to a successful grasp through various machine learning techniques.  

Considering multi-object grasps, in several works~\cite{Harada-ICRA-1998, Harada-IROS-1998, Harada-TRA-2000}, conditions for enveloping grasps of multiple objects using a multi-fingered robot hand, under the rolling contact assumption are proposed. Multiple objects (already in grasp wrench equilibrium) are grasped and lifted up towards the gripper's palm. \citet{Harada-ARK-2000} introduced active force closure for multiple objects. The work derives conditions to generate an arbitrary acceleration on multiple objects grasped with a multi-fingered robot hand. \citet{Yoshikawa-ICRA-2001} proposed a condition to achieve power grasps where they formulate the multi-object optimal power grasp problem and minimize finger joint torques. 

Other works~\cite{Yamada-ICRA-2005, Yamada-ISMNHS-2005,Yamada-ROBIO-2009,  Yamada-ICMA-2012} propose methods to evaluate the grasp stability of multiple planar objects grasped by a multi-fingered robot hand. Fingers are replaced with 2D spring models and grasp stability is analysed through potential energy stored in the grasp. \citet{Yamada-JCSE-2015} extended these multi-object grasp stability analyses to the 3D case, under both rolling and sliding contact. Results in these multi-object grasping works are mostly only numerical simulations, without real-robot multi-object grasps. In this paper, we derive conditions for equilibrium multi-object grasps, under the frictionless point-contact model. We also show real-robot multi-object grasps. 

Recently, \citet{Chen-IROS-2021} investigated the problem of dipping a robot hand inside a pile of identical spherical objects, closing the hand and estimating the number of objects remaining in the hand after lifting. \citet{Shenoy-CoRR-2021} also studied multi-object grasping in a similar setting of identical spherical objects, but with a goal of transferring the picked objects to another bin. Their approach generates a pre-grasp configuration and flexion strategy that corresponds to the desired number of objects to be picked from a given pile. Our work is focused on multi-object grasps of complex-shaped objects where physics predictions play a major role. 

At the heart of various push-grasping methods \cite{Dogar-RSS-2011, Agboh-Humanoids-2018} is the need to make physics-based predictions of how objects will move when pushed and potentially squeezed together. It is computationally expensive to predict the result of these contact interactions~\cite{Agboh-WAFR-2018, Agboh-arxiv-2021}. Thus, we propose necessary conditions for multi-object push-grasping for the first time to the best of our knowledge. We use these conditions in a novel multi-object grasp planner to synthesize multi-object grasps. Thereafter, we use our grasp planner in a novel picking system that generates both single and multi-object grasps and picks randomly placed objects. 
%
\section{Problem statement}
\label{sec:problem_statement}
\vspace{-3mm}
As shown in Fig.~\ref{fig:mog_intro}, we consider the problem where multiple rigid convex polygonal extruded objects rest in randomly placed positions and orientations on a planar surface visible
from an overhead camera. The goal is to develop a picking algorithm that efficiently grasps and transports all objects to a bin, using single and multi-object grasps. 

\noindent \textbf{Assumptions:}
We assume a frictionless point contact model only \textit{between} objects in a multi-object grasp. We assume antipodal multi-object grasps where each object is kept in equilibrium by two neighbouring objects, or one object and a gripper jaw. We assume a parallel-jaw gripper and objects that are extruded convex polygons with  uniform mass and known geometry.

\noindent \textbf{State and action:}
The state $\vec{x}$ pose of all objects:
$\vec{x} = [\{x_{i}, y_{i}, \theta_{i}\}], \hspace{1mm} \mathrm{for} \hspace{1mm} \mathrm{i} \hspace{1mm} \mathrm{=} \hspace{1mm} \mathrm{0}, \dots ,N_{o} - 1,
$
where $N_{o}$ is the number of objects on the table, $x_i$, $y_i$, and $\theta_{i}$ represent the 2D pose of object $i$, provided by an overhead camera. We represent single and multi-object grasp actions in the same way:  
    $\vec{u} = [x_{g}, y_{g}, \theta_{g}]$
, where $x_{g}, y_{g}$, and $\theta_{g}$ represent the desired grasp pose of the gripper, after which the jaws close with a maximum force $f_{g}$. The grasp action involves four steps. (1) Moving the open gripper above the initial desired grasp pose, and lowering down until just above the surface/table. (2) Closing the gripper jaws. (3) Moving the gripper upwards, off the table, and above the bin. (4) Opening the jaws such that objects fall into the bin.

\section{Multi-Object  Grasps}\label{sec:multi-object-grasps}
\vspace{-1mm}
To plan multi-object grasps, we need to understand conditions under which a multi-object grasp will succeed or fail. In this section, we study the conditions to obtain an equilibrium grasp for single and multiple planar, convex polygonal objects. Equilibrium grasps on an object occur when the grippers can apply forces and torques consistent with the contact model assumed at contact points, such that the net wrench on the object is zero. We study these grasps for multiple objects, under the frictionless point-contact model.
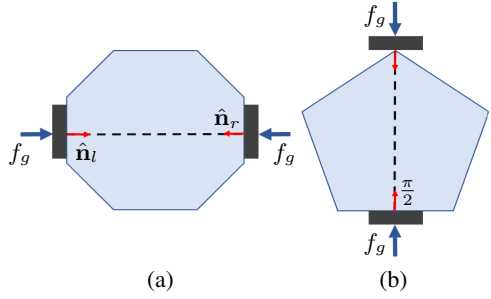
\begin{figure}[t]
    \centering
    \includegraphics[scale=0.4]{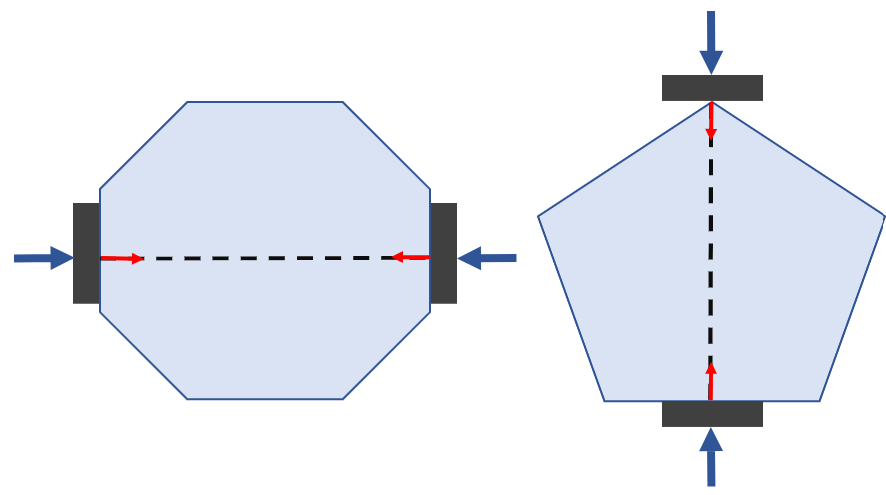}
    \begin{picture}(0,0)
    \put(-130,-10) {(a)}
    \put(-43,-10) {(b)}
    \put(-182, 38){$f_g$}
    \put(-157,38) {$\hat{\vec{n}}_{l}$}
    \put(-105,51) {$\hat{\vec{n}}_{r}$}
    \put(-36,24){$ \frac{\pi}{2}$}
    \put(-49, 4) {$f_g$}
    \put(-49, 90) {$f_g$}
    \put(-83,38){$f_g$}
    \end{picture}
    \vspace{2mm}
    \caption{Frictionless equilibrium grasps for single convex polygonal objects, with a parallel-jaw gripper. There are two possible grasp conditions in the frictionless case: (a) Two parallel edges. (b) Vertex and an opposing edge. The vertex has a perpendicular projection onto the edge.}
    \label{fig:stable-single-object}
    \vspace{-6mm}
\end{figure}
In single-object grasping, under this contact model, the contact force from the left (\textit{l}) and right (\textit{r}) gripper plates can be written as: $\vec{f}_{i} = f_{g} \hat{\vec{n}}_{i}, \hspace{1mm} i\in \{l, r\}$, 
%
%
where $\hat{\vec{n}}_{i}$ is the surface inward unit normal at the respective contact point, and $f_{g}$ is the magnitude of a maximum gripper force. We seek equilibrium grasps under the frictionless point contact model. Thus, contact forces from the left and right grippers must satisfy:
\begin{align}\label{eq:force_balance_single_frictionless}
    f_{g}\hat{\vec{n}}_{l} + f_{g} \hat{\vec{n}}_{r}  = \vec{0}, 
    \hspace{3mm}
    \vec{p}_{l} \times f_{g}\hat{\vec{n}}_{l} + \vec{p}_{r} \times f_{g} \hat{\vec{n}}_{r} = \vec{0},
\end{align}
%
where $\vec{p}_{l}$ and $\vec{p}_{r}$ are the position vectors of the left and right contact points respectively. 
Eq.~\ref{eq:force_balance_single_frictionless} implies that for parallel-jaw grippers on a convex polygonal object, we can achieve equilibrium grasps only if the two contact normals have opposing directions 
($\hat{\vec{n}}_{l} = -\hat{\vec{n}}_{r}$), and lie on the same line. These are antipodal grasps. 

Antipodal grasps on a convex polygonal object occur under any of the following two necessary conditions: (i) two parallel edges as shown in Fig.~\ref{fig:stable-single-object}a, and (ii) one vertex and an opposing edge where the perpendicular projection of the vertex lies inside  the corresponding edge, as shown in Fig.~\ref{fig:stable-single-object}b. There is a possible third case of grasping the convex polygonal object at two vertices, however that is unstable in the frictionless case and we do not consider such grasps. 

We extend the initial frictionless equilibrium grasp analysis to the multi-object case. 
To achieve an equilibrium grasp in the multi-object case, each object in the multi-object grasp must individually be in an antipodal grasp from contacts with grippers or neighbouring objects. In this work, we restrict our study to multi-object grasps where each object involved is kept in equilibrium by only two neighbouring objects or one object and a gripper plate. This implies that all contact normals across all objects must lie on the same line. We show a sample frictionless equilibrium multi-object grasp in Fig.~\ref{fig:stable-multi-object}. 

Specifically, from the single object equation, Eq.~\ref{eq:force_balance_single_frictionless}, the multi-object equilibrium grasp equations for a set of $n_{o}$ objects become: 
\begin{align}\label{eq:force_balance_multi_frictionless}
    f_{g}\hat{\vec{n}}_{l_i} + f_{g} \hat{\vec{n}}_{r_i}  = \vec{0}, \hspace{3mm} i \in \{0, 1,\dots,n_{o}-1\}
\end{align}
\begin{align}\label{eq:torque_balance_multi_frictionless}
  \vec{p}_{l_i} \times f_{g}\hat{\vec{n}}_{l_i} + \vec{p}_{r_i} \times f_{g} \hat{\vec{n}}_{r_i} = \vec{0}, \hspace{3mm} i \in \{0, 1, \dots, n_{o}-1\}, 
\end{align}
where $\hat{\vec{n}}_{l_i}$ and $\hat{\vec{n}}_{r_i}$ are the left and right contact normals for object $i$. For an equilibrium grasp, we find that $\hat{\vec{n}}_{l_i} = -\hat{\vec{n}}_{r_i}$, and both contact normals must lie on the same line for object $i \in \{0,1,\dots,n_{o}-1\}$. This also implies that all normals across all $n_{o}$ objects will lie on the same line. 
\begin{figure}[t]
    \centering
    \includegraphics[scale=0.46]{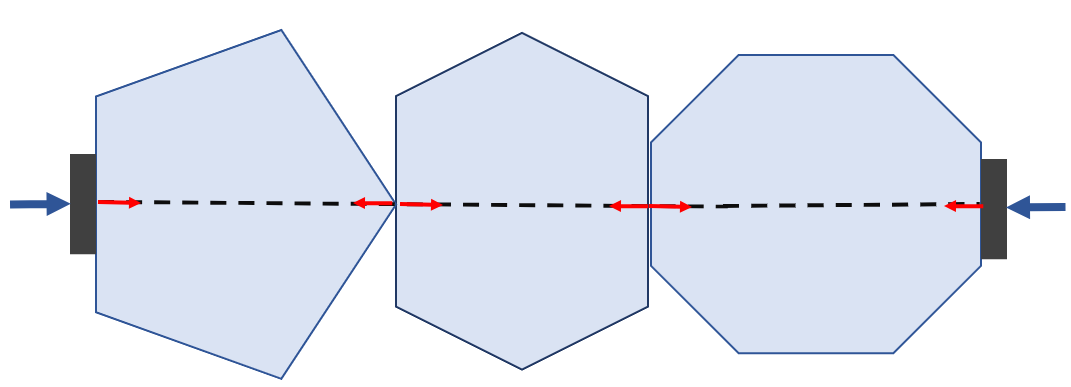}
    \begin{picture}(0,0)
    \put(-252,30) {$f_g$}
    \put(-225,32) {$\hat{\vec{n}}_{l_i}$}
    \put(-180,47) {$\hat{\vec{n}}_{r_i}$}
    \put(-155,32){$\hat{\vec{n}}_{l_{i+1}}$}
     \put(-123,47){$\hat{\vec{n}}_{r_{i+1}}$}
    \put(-95,32){$\hat{\vec{n}}_{l_{i+2}}$}
     \put(-49,47){$\hat{\vec{n}}_{r_{i+2}}$}
     \put(-8,32) {{$f_g$}}
    \end{picture}
    \caption{A sample frictionless equilibrium grasp for multiple convex polygonal objects. Each object is in an antipodal grasp and all contact normals lie on the same line.}
    \label{fig:stable-multi-object}
        \vspace{-7mm}
\end{figure}
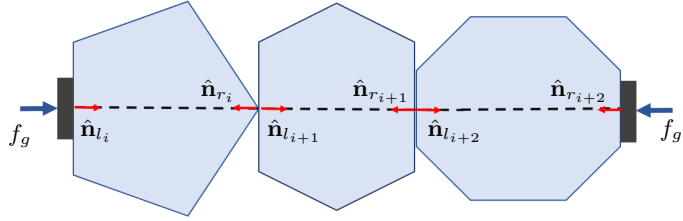

\section{Multi-Object Push Grasps}\label{sec:multi-object-push-grasps}
\vspace{-2mm}
Objects almost never start in contact, fully aligned and ready to be grasped as shown in the previous section and in other prior work \cite{Harada-TRA-2000, Yamada-JCSE-2015, Yoshikawa-ICRA-2001}. It is the case in practice that objects are usually apart and need to be squeezed together in a multi-object push-grasp. Hence, the need to make physics-based predictions of how objects will move as the robot pushes them together. 
\subsection{Two Necessary Conditions}
\vspace{-1mm}
Physics simulation for multi-object grasping is computationally expensive. We aim to avoid simulations and significantly speed-up multi-object grasp planning. In this section, we propose two necessary conditions for a multi-object push-grasp to succeed. We propose these conditions for the frictionless point-contact model. The conditions are: \textit{multi-object diameter function} and \textit{intersection area}. If any of these conditions is violated, the push-grasp is guaranteed to fail. We use these necessary conditions to filter out inadmissible grasps, and therefore to avoid computationally expensive simulations.
\subsection{Multi-Object Diameter Function} 
\vspace{-1mm}
\begin{figure}[t]
    \centering
    \includegraphics[scale=0.4]{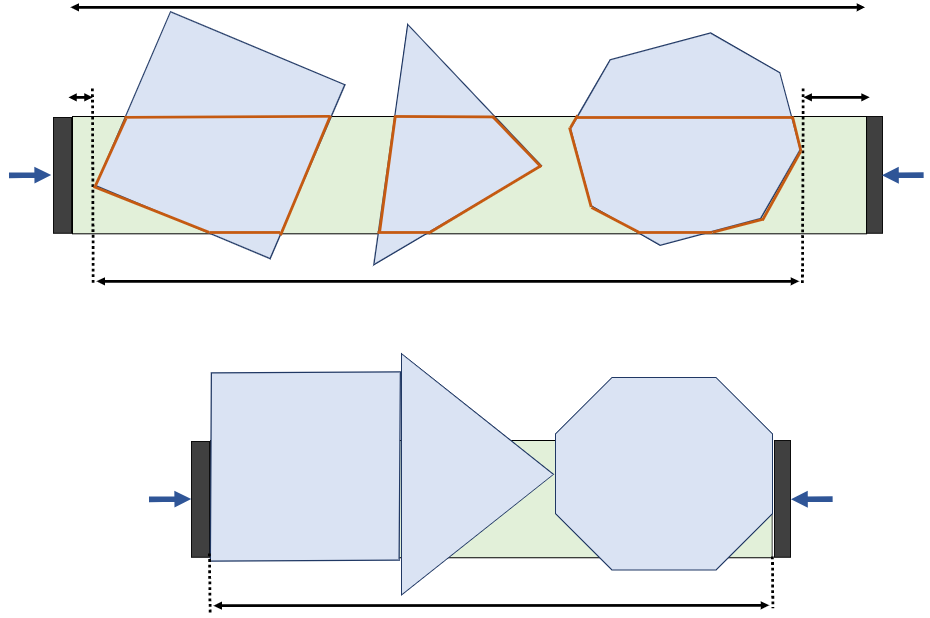}
    \begin{picture}(0,0)
    \put(-152, 87) {$o^{s}_{0}$}
    \put(-104, 87) {$o^{s}_{i}$}
    \put(-64, 87) {$o^{s}_{n_{0}-1}$}
    \put(-165,29) {$f_{g}$}
    \put(-24,29) {$f_{g}$}
    \put(-192,94) {$f_{g}$}
    \put(-9,94) {$f_{g}$}
    \put(-26,82) {$S$}
    \put(-110,59) {$h_{0}$}
    \put(-105,-7) {$h_f$}
    \put(-105,127) {$w_g$}
    \put(-178, 110){$b_l$}
    \put(-26, 110){$b_r$}
    \end{picture}
    \vspace{2mm}
    \caption{A frictionless multi-object grasp. We show the initial internal rectangular region $S$, the gripper stroke $w_g$, the left and right minimum gripper distances to objects, $b_{l}$ and $b_{r}$ respectively, the initial multi-object grasp diameter ($h_{0}$), the final multi-object grasp diameter ($h_{f}$), and $o^{s}_{i}$, the intersection polygon (in red) between $S$ and object $O_i$. We also show that $h_{0} \geq h_{f}$ for any equilibrium multi-object grasp.}
    \label{fig:frictionless-mopg}
        \vspace{-5mm}
\end{figure}
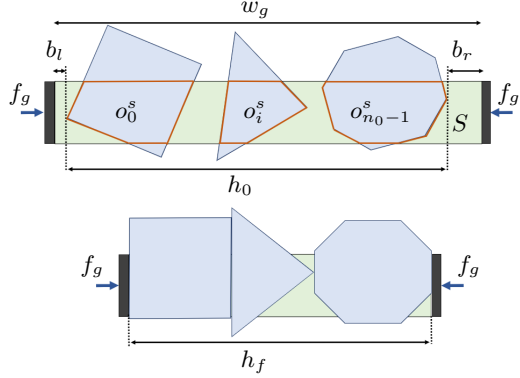

In single-object grasping, consider quasi-statically closing a parallel-jaw gripper such that both jaws simultaneously make the first contact with the object. The distance between the jaws at this initial contact is called the \textit{diameter}, first introduced by \citet{Goldberg-Algorithmica-1993} for frictionless grasping. Let this initial diameter be $d_{0}$. As we continue to close the grippers past this initial contact, the diameter must decrease until it reaches a local minimum value, at the stable grasp configuration. Let this final diameter be $d_{f}$. Thus, for any stable single-object grasp, we have that: $d_{0} \geq d_{f}$. Next, we extend the single-object diameter condition to the multi-object case. 

We begin by introducing the multi-object grasping diameter ($h$). We illustrate how to compute $h$ given a grasp $\vec{u}$ in Fig.~\ref{fig:frictionless-mopg}. We define the initial internal region between the parallel-jaw grippers to be the internal rectangular region, $S$. Let $o^{s}_{i} = S \cap O_{i}$, be the intersection polygon between $S$ and object $O_i$.
%
%
Let $w_g (t)$ be the gripper stroke at time $t$. Let $b_{l}(t)$ be the shortest distance between $o^{s}_{0}$ and the left jaw, and $b_{r}(t)$ be the shortest distance between $o^{s}_{n_{0}-1}$ and the right jaw. Then, the multi-object grasp diameter is:
\begin{align}
    h(t) = w_{g}(t) - (b_{l}(t) + b_{r}(t)). 
\end{align}

Let $h_{0}$ be the initial multi-object grasp diameter at time $t_{0}$, and let $h_{f}$ be the corresponding final multi-object grasp diameter, at time $t_{f}$ when the grippers become stationary after closing. $h_{0} = h(t_{0}) = w_{g}(t_{0}) - (b_{l}(t_{0}) + b_{r}(t_{0}))$, and $h_{f} = h(t_{f}) = w_{g}(t_{f}) - (b_{l}(t_{f}) + b_{r}(t_{f}))$. When the grippers become stationary at $t_{f}$, $b_{l}(t_{f}) = b_{r}(t_{f}) = 0$, and $h_{f} = w_{g}(t_{f})$. Since parallel-jaw grippers close in only one direction, it implies that $w_{g}(t_{i}) \geq w_{g}(t_{i+1})$, where $t_{i+1} > t_{i}$, and by extension $w_{g}(t_{0}) \geq w_{g}(t_{f})$. Therefore, any multi-object grasp will satisfy:
\begin{align}\label{eq:diameter-init-condition}
    h_{0} \geq h_{f}
\end{align}
In Fig.~\ref{fig:frictionless-mopg}, we show a sample multi-object initial diameter $h_{0}$ and final diameter $h_{f}$, for a frictionless multi-object grasp. The multi-object grasp diameter condition in Eq.~\ref{eq:diameter-init-condition} requires knowledge of $h_{f}$ which in turn requires a physics simulation to compute. This defeats the purpose of this condition. 

\begin{theorem}
Given a set of $n_{o}$ objects, there exists a constant multi-object grasp diameter, $h_{f_\mathrm{min}}$ such that any grasp $\vec{u}$, with some initial multi-object grasp diameter $h_{0}$ is guaranteed to fail when $h_{0} < h_{f_\mathrm{min}}$. 
\end{theorem}

\begin{proof}

Every polygonal object $i$ in a successful multi-object grasp of $n_{o}$ objects has a final diameter $d^{i}_{f}$, at time $t_{f}$ when the grippers become stationary after they close. This corresponds to a multi-object grasp diameter $h_{f} = \sum^{n_{0}-1}_{i=0} d^{i}_{f}$. This summation is true for any successful multi-object grasp because contact normals across all $n_{o}$ objects must lie on the same line (Eq.~\ref{eq:force_balance_multi_frictionless} and Eq.~\ref{eq:torque_balance_multi_frictionless}).

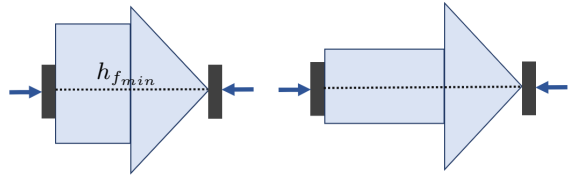
\begin{figure}[t]
    \centering
    \includegraphics[scale=0.4]{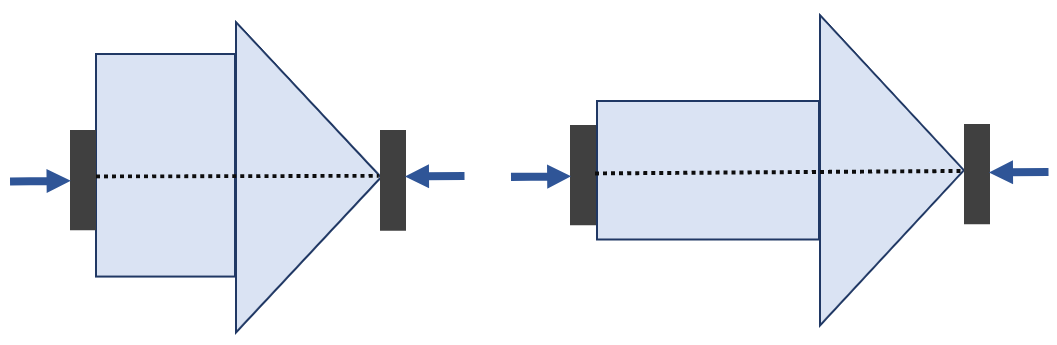}
    \begin{picture}(0,0)
    \put(-180,38) {$h_{f_{min}}$}
    \end{picture}
    \caption{ All possible unique final frictionless multi-object grasp diameters for a rectangle and an isosceles triangle. There are only two unique multi-object grasp configurations (excluding vertex-vertex contacts). The left configuration has the minimum multi-object diameter $h_{f_\mathrm{min}}$.}
    \label{fig:alls-hs}
        \vspace{-5mm}
\end{figure}

Every object has a minimum final diameter ($d^{i}_{f_{min}}$), at which a stable frictionless single-object grasp will occur. We can enumerate all final object grasp diameters for any convex polygonal object under the frictionless point-contact model using the two conditions shown in Fig.~\ref{fig:stable-single-object}. Let the finite number of possible final diameters be $N^{i}_{d_f}$ for a given object $i$. The minimum final diameter can then be written as: 
\begin{align}\label{eq:single-min-final-diameter}
    d^{i}_{f_\mathrm{min}} = \min (d^{i}_{f_{j}}), \hspace{3mm} j \in \{0,1, \dots, N^{i}_{d_{f}}-1\}. 
\end{align}

Then, we compute the minimum multi-object grasp diameter $h_{f_\mathrm{min}}$ at which a stable frictionless multi-object grasp of $n_{o}$ objects will occur as:
\begin{align} \label{eq:frictionless-multi-object-diameter}
    h_{f_\mathrm{min}} = \sum^{n_{o}-1}_{i=0} d^{i}_{f_{min}}.
\end{align}

Recall from Eq.~\ref{eq:diameter-init-condition} that $h_{0} \geq h_{f}$ for any multi-object grasp, but $d^{i}_{f} \geq d^{i}_{f_\mathrm{min}}$ which implies $h_{f} \geq h_{f_\mathrm{min}}$. Therefore any successful multi-object grasp must satisfy:
\begin{align}\label{eq:multi-object-diameter}
    h_{0} \geq h_{f_\mathrm{min}}.
\end{align}

Thus, if the initial multi-object grasp diameter is less than the minimum possible final multi-object grasp diameter ($h_{0} < h_{f_\mathrm{min}}$), the grasp $\vec{u}$ is guaranteed to fail. This means that $h_{0}$ is so small that as the gripper closes, a stable multi-object grasp diameter can never be reached, completing the proof. \qed
\end{proof}
 
 For example, consider the case in Fig.~\ref{fig:alls-hs}, where we have two objects such that $n_{o}=2$. The rectangular object ($O_0$) has two unique stable grasp configurations under frictionless grasping with a parallel-jaw gripper ($N^{0}_{d_{f}}=2$), disregarding the other configurations due to symmetry, and vertex-vertex grasps. Similarly, the triangular object ($O_1$) has only one unique stable grasp configuration ($N^{1}_{d_{f}}=1$), again disregarding vertex-vertex grasps. Thus, there are two unique multi-object grasp configurations in the frictionless case and the left configuration has the minimum multi-object grasp diameter $h_{f_\mathrm{min}}$. We show an example of a grasp that failed due to a violation of the multi-object grasp diameter condition in Fig.~\ref{fig:mogviolations}(top). 

\subsection{Intersection Area} 

Each object in a multi-object equilibrium grasp needs to also be in a single-object equilibrium grasp. This requires at least two contacts, whether from the parallel jaws or other objects. Thus, as the gripper closes, if we predict no contacts for at least one object, the multi-object grasp fails. 

Let $A_{i}(t) = \mathrm{Area}(o^{s}_{i})$, be the area of the intersection polygon for object $i$, during a multi-object grasp.
Then, we require the intersection area for each object be greater than zero when the gripper closes.  $A_{i}(t)$ changes as the gripper closes with constant force $f_{g}$, and hence is a function of time. Thus, the intersection area necessary condition for $n_{o}$ objects become:
\begin{align}\label{eq:intersection_area}
 A_{i} (t_{f}) > 0, \hspace{1mm} i \in \{0,1,\dots, n_{o}-1\}
\end{align}
where $t_f > 0$ is when the gripper becomes stationary after closing. Knowing exactly how $A_{i} (t)$ changes as the gripper closes requires a physics simulation, which defeats the purpose of this condition. However, observe that if the starting intersection area for an object is zero it is likely to remain so as the gripper closes. This is especially true when the object of interest is sufficiently far away from $S$. 


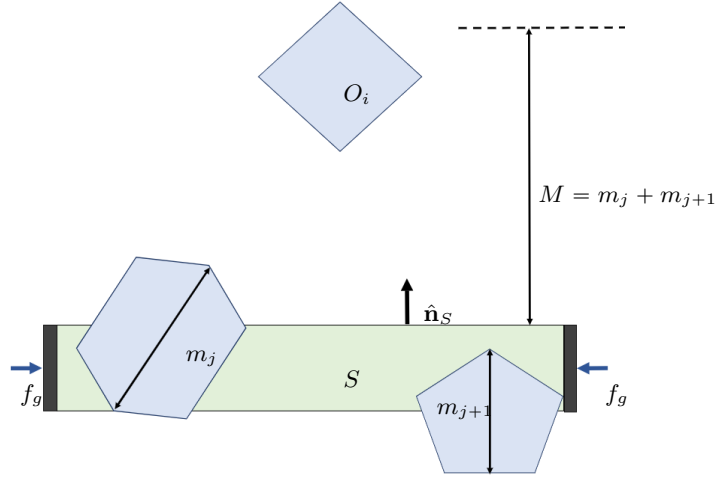
\begin{figure}[t]
    \centering
    \includegraphics[scale=0.56]{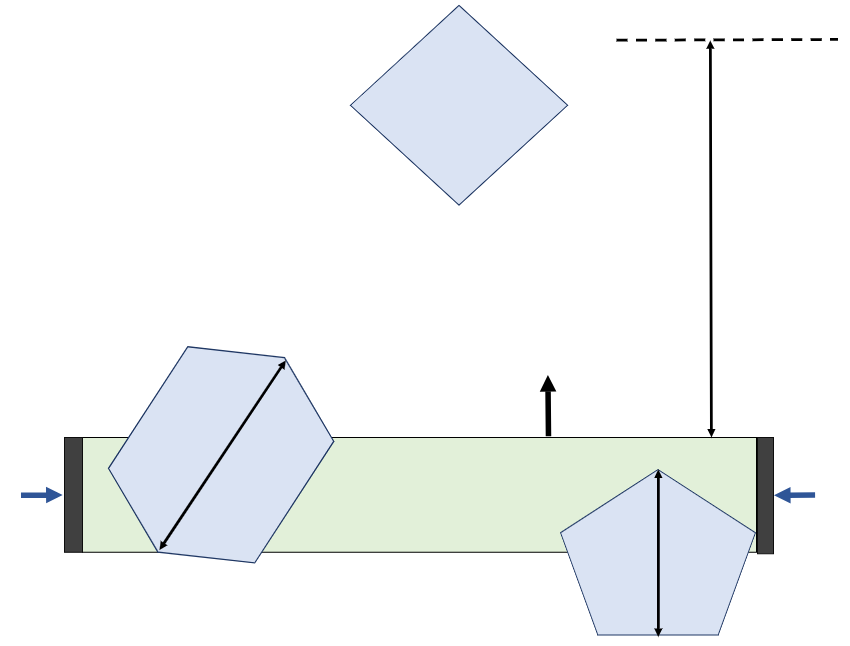}
    \begin{picture}(0,0)
    \put(-170,45) {$m_{j}$}
    \put(-111, 143) {$O_{i}$}
    \put(-76,25) {$m_{j+1}$}
    \put(-38,105) {$M=m_{j}+m_{j+1}$}
    \put(-81,60){$\hat{\vec{n}}_{S}$}
     \put(-232,30) {$f_{g}$}
     \put(-13,30) {$f_{g}$}
     \put(-111,35){$S$}
    \end{picture}
    \caption{An illustration of the intersection area necessary condition's proof. $M=\sum^{n_{o}-1}_{j=1}m_{j}$ is an upper bound on $\epsilon$, the minimum distance between $O_{i}$ and $S$ such that the multi-object grasp is guaranteed to fail when $A_{i}(0)=0$.}
    \label{fig:int-area-proof}
    \vspace{-5mm}
\end{figure}

\begin{theorem}
Given an object $O_{i}$ in a multi-object grasp $\vec{u}$ of $n_{o}$ objects, if $A_{i}(0)=0$ then there exists some minimum distance $\epsilon > 0$ between $O_{i}$ and $S$ such that the multi-object grasp is guaranteed to fail if the distance between $O_i$ and $S$ is larger than $\epsilon$.
\end{theorem}

\begin{proof}
Consider Fig.~\ref{fig:int-area-proof}. Object $O_i$ moves towards $S$ and can subsequently achieve $A_{i}(t_{f}) > 0$ only when there is contact between it and other objects as the gripper closes. $\epsilon$ is the minimum distance between $O_{i}$ and $S$ where surrounding objects cannot make contact as the gripper closes with grasp $\vec{u}$. Let $m_{j}$ for a surrounding object $j$ be the length of its longest diagonal, and let $M=\sum^{n_{o}-1}_{j=1}m_{j}$ be the sum of such diagonals for all other objects. This is the farthest other ($n_{o}-1$) objects  can be lined up along the normal direction ($\hat{\vec{n}}_{S}$) of $S$ to reach $O_i$. 
Assuming quasi-static grasping, other objects cannot make contact with $O_{i}$ beyond $M$. It is an upper bound on $\epsilon$ such that $\epsilon \leq M$. Therefore if $O_{i}$ is at least $M$ minimum distance away from the internal rectangular region and $A_{i}(0)=0$, the multi-object grasp, including the object $O_i$ is guaranteed to fail, which completes the proof. \qed
\vspace{-3mm}
\end{proof}
In practice, $M$ would be too large to provide useful grasp filtering. 
%
Recall that there is a minimum grasp diameter ($h_{f_\mathrm{min}}$) at which a multi-object grasp succeeds. This means that starting at $h_{0}$, the grippers can only move a maximum distance of $\Delta{h_\mathrm{max}} = h_{0} - h_{f_\mathrm{min}}$, otherwise the grasp fails. Under quasi-static grasping and given that the motion of objects in $S$ are only due to the gripper motion, we assume each surrounding object will move at most by $\Delta{h_\mathrm{max}}$ 
along $\hat{\vec{n}}_S$ and
towards $O_{i}$.
Given, that $\Delta{h}_\mathrm{max}$ is small for a standard parallel-jaw gripper, we assume that contacts from surrounding objects $O_{j}$ will be too short to bring object $O_{i}$ into the internal rectangular region $S$.  Therefore, we approximate $\epsilon=0$, such that the relaxed intersection area condition becomes:
\begin{align}\label{eq:relaxed_int_area}
 A_{i} (0) > 0, \hspace{1mm} i \in \{0,1,\dots, n_{o}-1\}.
 \vspace{-3mm}
\end{align}
For example, take Fig.~\ref{fig:int-area-proof} where the multi-object grasp diameter condition is satisfied. Since $A_{i}(0)=0$, the grasp is guaranteed to fail with the relaxed condition in Eq.~\ref{eq:relaxed_int_area}. We show an example of a multi-object grasp that failed due to a violation of the intersection area necessary condition in Fig.~\ref{fig:mogviolations} (bottom). 

\section{Picking System}
\label{sec:picking_system} 
\vspace{-2mm}
This section details a picking system that uses both single and multi-object grasps to pick randomly placed objects. 
%
%
\subsection{Picking Algorithm}
\begin{figure*}[t]
		\begin{subfigure}[b]{0.19\textwidth}
	 \centering 
		\includegraphics[scale=0.25]{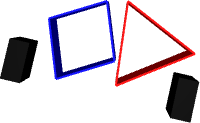}
	\end{subfigure}
		\begin{subfigure}[b]{0.19\textwidth}
	 \centering 
		\includegraphics[scale=0.25]{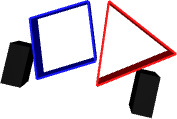}
	\end{subfigure}
		\begin{subfigure}[b]{0.19\textwidth}
	 \centering 
		\includegraphics[scale=0.25]{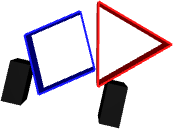}
	\end{subfigure}
	   \begin{subfigure}[b]{0.19\textwidth}
	 \centering 
		\includegraphics[scale=0.25]{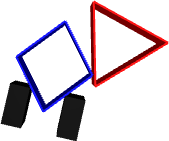}
	\end{subfigure}
	\begin{subfigure}[b]{0.19\textwidth}
	 \centering 
		\includegraphics[scale=0.25]{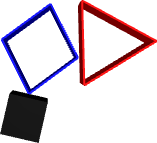}
	\end{subfigure}

	\begin{subfigure}[b]{0.19\textwidth}
	 \centering 
		\includegraphics[scale=0.25]{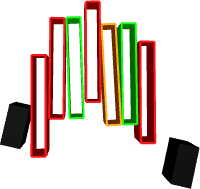}
	\end{subfigure}
		\begin{subfigure}[b]{0.19\textwidth}
	 \centering 
		\includegraphics[scale=0.25]{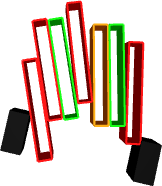}
	\end{subfigure}
		\begin{subfigure}[b]{0.19\textwidth}
	 \centering 
		\includegraphics[scale=0.25]{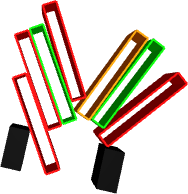}
	\end{subfigure}
	   \begin{subfigure}[b]{0.19\textwidth}
	 \centering 
		\includegraphics[scale=0.25]{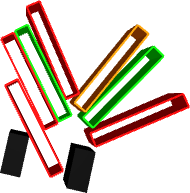}
	\end{subfigure}
	\begin{subfigure}[b]{0.19\textwidth}
	 \centering 
		\includegraphics[scale=0.25]{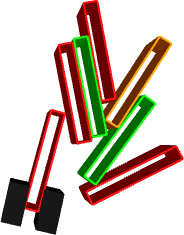}
	\end{subfigure}
	 \caption{Two examples of frictionless multi-object grasps that failed due to a violation of the two conditions. \textbf{Top row}: Multi-object diameter failure.  \textbf{Bottom row}: intersection area failure. The black rectangles are the parallel jaws of the gripper.}
	  \label{fig:mogviolations}
\end{figure*}
We present a picking algorithm in Alg.~\ref{alg:picking-algorithm}. Given a current state $\vec{x}$ containing $N_{o}$ objects (line 2), the algorithm finds a multi-object grasp for the largest group of objects that can be grasped together, using the GraspPlanner(.) subroutine (line 6) and executes the grasp, until no objects are left on the table or a time limit is reached (line 10). The CreateObjGroups(.) subroutine (line 3) loops through center points of objects and creates groups of all objects whose centers are within half a gripper width radius. Thereafter it eliminates all object groups that have a superset in the object groups list, apart from the single objects. The RankObjGroups(.) subroutine (line 4) ranks the list of object groups by their size, such that the grasp planner attempts to find a grasp for the largest object groups first. We provide details of the GraspPlanner(.) subroutine in the subsequent section. A grasp execution can result in failure where desired objects escape the grippers and other objects can be displaced from their original poses. The online nature of  Alg.~\ref{alg:picking-algorithm}, permits finding new grasps until the table is cleared or a time limit is reached. 
\subsection{Frictionless Multi-Object Grasp Planner (GP)}
\label{sec:grasp-planner}
\vspace{-3mm}
In this section, we propose a multi-object grasp planner. It is summarized in Alg.~\ref{alg:grasp-planner}.
The GraspPlanner(.) returns a grasp $\vec{u}$ if it can find one for the current group of objects, $obj{\_}group$, given the current state. 

\setlength{\textfloatsep}{2mm}
\begin{algorithm}[b]
    \SetKwInOut{Input}{Input}
    \SetKwInOut{Output}{Output}
    \SetKwInOut{Parameters}{Parameters}
    \SetKwInOut{Subroutines}{Subroutines}
    \Do{$N_{o}$ $>$ $\mathrm{0}$ $\mathbf{and}$ $\mathrm{time}$ $\mathrm{remaining}$}{ 
    $\vec{x}$, $N_{o}$ $\gets$ GetCurrentState(.)\\
    obj_groups $\gets$ CreateObjGroups($\vec{x}$)\\
    ranked_obj_groups $\gets$ RankObjGroups(obj_groups) \\

    \For{obj_group $\mathbf{in}$ $\mathrm{ranked\_obj{\_}groups}$}{
    $\vec{u}$ $\gets$ GraspPlanner($\vec{x}$, \textit{obj_group}) \\
    \If{ $\vec{u}$ $!$$=$ \{\}} 
    { 
     Execute $\vec{u}$\\
     break
    }
    }}
    \caption{Picking Algorithm}\label{alg:picking-algorithm}
\end{algorithm}
\begin{algorithm}[t]
    \SetKwInOut{Input}{Input}
    \SetKwInOut{Output}{Output}
    \SetKwInOut{Parameters}{Parameters}
    \SetKwInOut{Subroutines}{Subroutines}
    \Input{$\vec{x}$: Current state,  
    obj_group: Objects in the potential grasp}
    \Output{$\vec{u}$: A grasp action}
    
    grasp_cands $\gets$ GenGraspCands($\vec{x}$, obj_group) \\ 
    ranked_grasp_cands $\gets$ RankGraspCands(grasp_cands) \\
    \For{$\vec{u}$ $\mathbf{in}$ $\mathrm{ranked\_grasp\_cands}$}{
    \If {$\mathbf{not}$ $\mathrm{GraspFailure}$($\vec{x}$, $\vec{u}$)}{ 
        \If {$\mathrm{CheckGraspSuccess}$($\vec{x}$, $\vec{u}$)} {\Return {$\vec{u}$}}
    }
    }\Return \{\} \\
    \caption{Grasp Planner (GP)}\label{alg:grasp-planner}
\end{algorithm}
\setlength{\floatsep}{2mm}
First, the algorithm generates multiple grasp candidates (line 1). Second, it ranks these grasp candidates based on their likelihood of being successful (line 2). Finally, we sequentially test grasps to find a successful one (lines 3--6). To test a multi-object grasp we first use a grasp filtering system. We check if the grasp satisfies the two neccessary conditions outlined in Sec.~\ref{sec:multi-object-push-grasps} on line 4. If the necessary conditions are satisfied, we check for grasp success using a physics simulator (line 5). We explain the main subroutines in the following paragrasphs:\\
\noindent\textbf{GenGraspCands( ):}
Recall that a grasp action is parametrized as the position and orientation of the gripper: $\vec{u} = [x_{g}, y_{g}, \theta_{g}]$. The gripper reaches the desired pose and closes. Given a group of objects, we first find the convex hull of all of the objects. Then, we generate $N_{p}$  position samples that uniformly cover the convex hull. At each position sample, we generate $N_{\theta}$ orientation samples. We reject samples that result in collisions, between the jaws and any object --- inside or outside the group.\\
\noindent\textbf{RankGraspCands( ):} We aim to find a parameter that quickly predicts the likelihood of a multi-object push-grasp success and use it to rank grasp candidates. 
Recall the two necessary conditions for multi-object push-grasp success in Sec.~\ref{sec:multi-object-push-grasps} --- multi-object diameter and intersection area. These conditions suggest that the larger the area of the intersection polygon for each object, the more likely a grasp will be successful.  
Therefore, we propose to use the total intersection area ($A_{T} = \sum^{n_{o}-1}_{i=0} A_{i} $) as a heuristic to judge how successful a grasp will be. The higher the total area, the more likely the grasp will be successful. $A_{i}$ is computed as the area of intersection between object $O_i$ and $S$.\\
\noindent\textbf{GraspFailure( ):} We filter multi-object grasps using the two conditions provided in Sec.~\ref{sec:multi-object-push-grasps} --- multi-object grasp diameter and intersection area. If any of the two necessary conditions is not satisfied, the multi-object grasp is guaranteed to fail. \\
\noindent\textbf{CheckGraspSuccess( ):} To check the success of a grasp $\vec{u}$, we use a physics simulator where we assume objects are rigid, and have a uniform mass distribution. We close the gripper at the specified pose, then lift it above the table to check that the grasp is stable. Using a physics simulator means that our grasp planning system has no false positives for simulation experiments. However, we expect some grasps to fail in physical experiments due to uncertainty in state and model parameters and aleatory uncertainty related to friction, sensing, and control.
\section{Experiments}
\vspace{-1mm}
We conduct simulation and physical experiments to evaluate the multi-object push-grasping necessary conditions, grasp planner, and picking system. 
In the following subsections, we explain the general setup, experimental details, baselines, and results.  

\subsection{Simulation and Physical Experiments:}
\vspace{-1mm}
We use Mujoco 2.1.0 as the physics simulator. In all experiments, we use a fixed set of objects, shown in Fig.~\ref{fig:mog_intro}. It contains a total of 33 objects from 3-sided to 6-sided convex polygons. We create a model of these objects in Mujoco where we ensure that objects are rigid and have a uniform mass distribution. In physical experiments, we use an RGB-D camera to estimate the pose of objects on a table. We approximate the frictionless case by selecting low ($\mu=0.01$) friction parameters in Mujoco. 

\subsection{Evaluation of Necessary Conditions in Simulation:}
\label{sec:exp_mog_necessary_conditions}
\vspace{-1mm}
In this section, we conduct simulation experiments to evaluate the proposed necessary conditions for multi-object push-grasping. We use a physics simulator to get ground-truth multi-object push-grasping results. If a necessary condition is violated, the grasp is predicted to fail (negative). If the ground truth also fails, we term this a true negative, otherwise it is a false negative. An important question here is what percentage of true negatives can we predict using these necessary conditions? Are there any false negatives? These questions will specifically help us evaluate the practicality of our intersection area assumption in Sec.~\ref{sec:multi-object-push-grasps}.

\begin{figure}[t]
\centering
\includegraphics[scale=0.5]{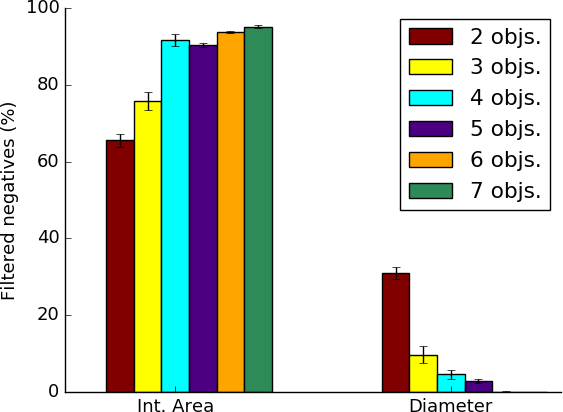}
\caption{Simulation experimental evaluation of the necessary conditions for multi-object push-grasping. Using both the intersection area (Int. Area) and the multi-object diameter (Diameter) necessary conditions, we can predict on average across all objects \textbf{93.4\%} of true-negative grasps in any given set of grasp candidates. The average number of grasp candidates in any given scene is \textbf{193.19 $\pm$ 1.54}. All errors are within 95\% confidence interval of the mean.}
	\label{fig:neccessary-conditions-results}
\end{figure}
We consider classes of 2-object to 7-object grasps (limited only by gripper width). For each multi-object grasp class, we create $200$ randomly generated scenes. That is a total of 1200 grasping scenes. Thereafter, for each grasp scene, we generate $N_{p} \times N_{\theta} = 70 \times 7 $ grasp candidate samples, and reject samples in collision. 
%
For each grasp candidate, we check both multi-object push-grasping necessary conditions and compare results to a Mujoco simulation of the grasp. We record any true negatives and any false negatives, across all grasp candidates, grasp scenes, and classes. 

%
Results on evaluating the necessary multi-object push-grasping conditions can be found in Fig.~\ref{fig:neccessary-conditions-results}. 
We did not record any false negatives across any object class---a result that supports the mild intersection area assumption.
A majority of the true-negative grasps (\textbf{85.4\%} on average) were predicted by the intersection area condition, while the multi-object diameter function predicted \textbf{8\%}. Both conditions combined predicted \textbf{93.4\%} of true negatives, with a \textbf{0.5\%} overlap between their predictions only for the 2-object case. We find that the multi-object grasp diameter condition violations reduce with increasing number of objects (from \textbf{31\%} for two objects), while violations for the intersection area increased. See Fig.~\ref{fig:mogviolations} for grasps that failed due to a violation of the multi-object diameter condition and the intersection area condition. 
%
\subsection{Grasp Planning Experiments in Simulation}
We conduct experiments in simulation to evaluate the multi-object grasp planner. Our goal here is to understand the effects of the grasp filtering and ranking systems. We follow the same process as Sec.~\ref{sec:exp_mog_necessary_conditions} to generate random scenes and corresponding grasp candidates. In each scene, we use the grasp planner (GP) and three baselines to find a grasp. 
We record the grasp planning time --- total time to return a grasp or report that no successful grasp was found. We also record the number of grasp samples tested/considered before a grasp is found. 

\noindent\textbf{Baseline 1 (Rand-Phys):}
This is a version of the multi-object grasp planner where we randomize the grasp candidates after they are generated (no grasp ranking). Then test each grasp sequentially in only the physics simulator (no grasp filtering), until a grasp is found or all grasp candidates are tested. 

\noindent\textbf{Baseline 2 (Rank-Phys):}
This is a version of the multi-object grasp planner that ranks grasps but tests grasps sequentially using only the physics simulator (no grasp filtering). 

\noindent\textbf{Baseline 3 (Rand-Fil-Phys):} This randomizes the grasp candidates (no ranking) as in baseline 1, but uses grasp filtering and the physics simulator.

\noindent\textbf{Results:} 
We show simulation multi-object grasp planning experimental results in Fig.~\ref{fig:grasp-planning-results}. Planning time results are in Fig.~\ref{fig:gp_planning_time}, results on grasps tested in the physics simulator are in Fig.~\ref{fig:gp_tested_cands}.
Our proposed approach (GP) has the lowest planning time of \textbf{2.2 seconds} on average which is \textbf{19 times faster} than directly using only a physics simulator (Rand-Phys). We note that the physics simulator models the full dynamics of the system but we do not, and this explains the speedup.
We find that the grasp ranking method reduced the number of tested grasps by \textbf{93.44\%}, when we compare Rand-Phys and Rank-Phys.  
\begin{figure*}[t]
\centering
	\begin{subfigure}[b]{0.497\textwidth}
	\centering
		\includegraphics[scale=0.43]{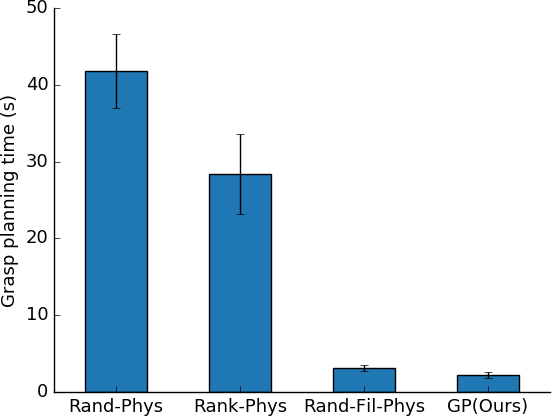}
		\caption{Planning time.}
		\label{fig:gp_planning_time}
	\end{subfigure}
	\begin{subfigure}[b]{0.497\textwidth}
	 \centering 
		\includegraphics[scale=0.43]{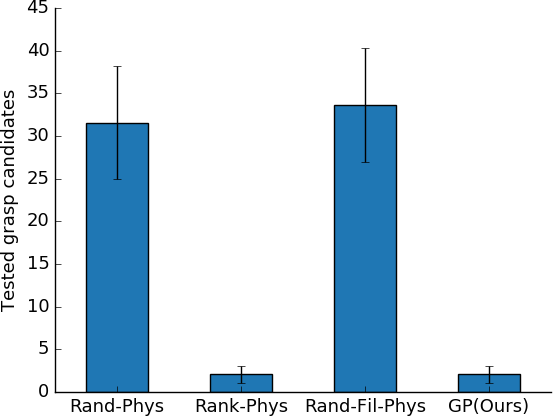}
		\caption{Number of tested grasps.}
		\label{fig:gp_tested_cands}
	\end{subfigure}
	\caption{Experimental results in simulation on multi-object grasp planning. Our approach (GP) found a grasp or reported that none exists in  \textbf{2.2 seconds} on average which is \textbf{19 times faster} than the Mujoco physics simulator baseline (Rand-Phys), thanks to a fast grasp filtering system. GP tests an average of 2 candidate grasps before a successful grasp is found. That is \textbf{16.2 times less} candidate grasps using the ranking system, compared to a random approach (baselines Rand-Phys, and Rand-Fil-Phys). Error bars are shown within 95\% confidence interval of the mean.}
	\label{fig:grasp-planning-results}
\end{figure*}

\subsection{Physical Experiments}
We conduct physical picking experiments to evaluate the picking system presented in this work, in comparison to a single object picking system. Details of this single-object picking system can be found below. 

\noindent\textbf{Single object picking system:}
We use the same multi-object picking algorithm and grasp planner for single object picking. However, we restrict the number of objects to a randomly chosen single object.

\noindent\textbf{Metrics:} We compare both picking systems based on success rate, percentage of objects picked, and picks per hour.  \textit{Success rate}: the percentage of grasp attempts that moved at least one object into the box. \textit{Percent picked}: the fraction of objects that were moved to the box. \textit{Picks per hour (PPH)}: total number of objects picked per hour, through single or multi-object grasps. 

\noindent\textbf{Physical experiment details:}
We use the UR5 robot arm with a Robotiq 2F-85 gripper, and an Astra RGBD camera in all experiments. We create 20 picking scenes where we randomly place 33 objects in different positions and orientations. We clustered objects arbitrarily to create opportunities for multi-object grasps. In each scene, we use both the multi-object picking system and the single-object picking system to generate grasps. We replicated the scene manually between the two picking systems, leading to a total of 40 real-robot runs. Please see Fig.~\ref{fig:mog_intro} for a sample real-world scene. A failed grasp attempt is where the robot misses a grasp (all objects escape), or where all objects fall out of the gripper before they reach the box. 

\noindent\textbf{Results:}
Picking results can be found in Table.~\ref{table:picking_results}. Both systems picked a similar percentage of objects. However, the multi-object grasping system achieved \textbf{13.6\%} higher grasp success. A major source of grasp failure is uncertainty in state and model parameters. This affects success rates of both picking systems. The single-object system suffers more as it makes more grasp attempts aimed at picking single objects that are oftentimes not easily accessible due to clutter. Another source of grasp failure is model mismatch. Grasps were planned in the physics simulator and the underlying assumptions there may not hold on the real physical system. We find that the multi-object grasping system is \textbf{59.9\%} faster than the single object picking baseline in picks per hour (PPH). This is in spite of the fact that uncertainty can sometimes aid the single-object picking system in grasping more than one object at a time. 
\vspace{-8mm}
\begin{center}
\begin{table}[t]
\caption{Physical picking experimental results with 20 scenarios, each with 33 objects, placed in different random positions and orientations.} 
\centering 
\begin{tabular}{@{}l@{\quad}c@{\quad}c@{\quad}c@{}} 
\toprule 
Methods & Success rate (\%) & Percent picked & PPH\\ 
\midrule
Single-Object  & 73.31 $\pm$ 4.84  & 97.77 $\pm$ 1.42 & 212.59 $\pm$ 19.03 
\\
Multi-Object & 83.33 $\pm$ 4.29 & 99.52 $\pm$ 0.48 & 340.08 $\pm$ 27.09 
 \\
\bottomrule
\end{tabular}
\label{table:picking_results}
\end{table}
\end{center}
%
\section{Discussion and future work}
\vspace{-2mm}
We propose the planar frictionless multi-object grasping problem where multiple convex polygonal objects are grasped and transported to a bin. We provide necessary conditions for multi-object push grasping under the frictionless point contact model, and their corresponding theorems and proofs. Experiments in simulation and on the physical robot suggest a significant speed-up in grasp planning time compared to a physics simulator baseline. We also find a reduction in overall picking time, compared to a single-object picking system. 

One assumption in this work is the frictionless point contact model between objects. This limits the number of feasible multi-object grasps. In future work, we will consider friction and how it affects multi-object grasps. We planned grasps in a physics simulator and without considering uncertainty. This lead to grasp failure. We will extend this work to generate robust multi-object grasps and to consider general 3D objects, including deformables. We restricted our study to antipodal multi-object grasps. In future work, we will consider other forms of multi-object grasps, including multi-fingered grasps. 

\section{Acknowledgement}
\vspace{-2mm}
This research was performed at the AUTOLAB at UC Berkeley in affiliation with the Berkeley AI Research (BAIR) Lab, and the CITRIS "People and Robots" (CPAR) Initiative. The authors were supported in part by donations from Siemens. Mehmet Dogar was partially supported by an EPSRC Fellowship (EP/V052659) and Wisdom C. Agboh was supported by EPSRC Doctoral Prize Fellowship Award EP/T517860/1. Any opinions, findings, and conclusions or recommendations
expressed in this material are those of the author(s) and do not necessarily reflect the views of the Sponsors.
\bibliographystyle{unsrtnat}
\bibliography{mog-bib}

\end{document}